\documentclass[11pt]{article}
\PassOptionsToPackage{numbers}{natbib}

\usepackage[preprint]{neurips_2018}

\usepackage{hyperref} 
\usepackage[ruled,lined]{algorithm2e}
\usepackage{graphicx}

\usepackage{booktabs}       
\usepackage{amsfonts}       
\usepackage{amsmath}
\usepackage{amsthm}

\newtheorem{theorem}{Theorem}
\newtheorem{definition}[theorem]{Definition}
\newtheorem{lemma}[theorem]{Lemma}

\newcommand{\ignore}[1]{}
\newcommand{\setR}{\mathbb{R}}
\newcommand{\einsum}{\texttt{einsum} }
\newcommand{\id}{\mathbb{I}}

\DeclareMathOperator{\diag}{diag}
\newcommand{\norm}[1]{\left\Vert#1\right\Vert}
\newcommand{\bnorm}{\Big\Vert}
\newcommand{\lnorm}[1]{\big\Vert#1\big\Vert}

\begin{document}

\title{A Simple and Efficient Tensor Calculus for Machine Learning}

\author{
  S\"oren Laue\\
  Friedrich-Schiller-Universit\"at Jena\\
  \&\\
  Data Assessment Solutions GmbH\\
  \texttt{soeren.laue@uni-jena.de} \\
  \And Matthias Mitterreiter\\
  Friedrich-Schiller-Universit\"at Jena\\
  Germany\\
  \texttt{matthias.mitterreiter@uni-jena.de} \\
  \And Joachim Giesen\\
  Friedrich-Schiller-Universit\"at Jena\\
  Germany\\
  \texttt{joachim.giesen@uni-jena.de}
}
\maketitle


\ignore{
\author{S\"oren Laue \and Matthias Mitterreiter \and Joachim Giesen\\
Friedrich-Schiller-Universit\"at Jena\\
Faculty of Mathematics and Computer Science\\
Ernst-Abbe-Platz 2\\
07743 Jena, Germany\\
\{soeren.laue, matthias.mitterreiter, joachim.giesen\}{@}uni-jena.de} \maketitle
}

\begin{abstract}
  Computing derivatives of tensor expressions, also known as tensor
  calculus, is a fundamental task in machine learning. A key concern
  is the efficiency of evaluating the expressions and their
  derivatives that hinges on the representation of these
  expressions. Recently, an algorithm for computing higher order
  derivatives of tensor expressions like Jacobians or Hessians has
  been introduced that is a few orders of magnitude faster than
  previous state-of-the-art approaches. Unfortunately, the approach is
  based on Ricci notation and hence cannot be incorporated into
  automatic differentiation frameworks from deep learning like TensorFlow, PyTorch,
  autograd, or JAX that use the simpler Einstein notation. This leaves
  two options, to either change the underlying tensor representation
  in these frameworks or to develop a new, provably correct algorithm
  based on Einstein notation. Obviously, the first option is
  impractical. Hence, we pursue the second option. Here, we show that
  using Ricci notation is not necessary for an efficient tensor
  calculus and develop an equally efficient method for the simpler
  Einstein notation. It turns out that turning to Einstein notation
  enables further improvements that lead to even better efficiency.
  
  The methods that are described in this paper have been implemented in the online tool \url{www.MatrixCalculus.org} for computing derivatives of matrix and tensor expressions.
  
  An extended abstract of this paper appeared as ``A Simple and Efficient Tensor Calculus'', AAAI 2020~\cite{LaueMG20a}.
\end{abstract}

\section{Introduction}

Many problems in machine learning are naturally written in terms of
tensor expressions. Any algorithmic method for computing derivatives
of such expressions is called a tensor calculus.  Standard automatic
differentiation (deep learning) frameworks like TensorFlow~\cite{tf},
PyTorch~\cite{pytorch}, autograd~\cite{Maclaurin15}, and
JAX~\cite{JAX} are very efficient when computing derivatives of
scalar-valued functions. However, evaluating the derivatives of
non-scalar-valued functions, for instance, Jacobians or Hessians, in
these frameworks is up to three orders of magnitude slower than
evaluating the derivatives that are computed by the approach
of Laue et al.~\cite{LaueMG2018}.

There have been some recent attempts to alleviate this lack of
efficiency by accelerating the underlying linear algebra using
automatic batching and optimizing the computational graphs of the
derivatives~\cite{XLA}. These improvements have been incorporated into
TensorFlow and JAX. However, the improvements are rather small and the
efficiency gap of up to three orders of magnitude still persists.

On the other hand, the approach of Laue et al.~\cite{LaueMG2018} relies crucially
on Ricci notation and therefore cannot be incorporated into standard
deep learning frameworks that use the simpler Einstein notation. Here,
we remove this obstacle and provide an efficient tensor calculus in
Einstein notation. Already the simple version of our approach is as
efficient as the approach by Laue et al.~\cite{LaueMG2018}. We provide further
improvements that lead to an even better efficiency.

Ricci notation distinguishes between co- and contravariant indices,
that is, upper and lower indices. This distinction is necessary in
order to compute derivatives in a mathematical correct way. Consider
for instance the simple expression $x^\top Ax$. If we want to compute
the derivative of this expression with respect to the vector $x$,
then, at some point, we face the problem of computing the derivative
of $x^\top$. However, this derivative in Ricci notation is the
delta-tensor $\delta_{ij}$ that cannot be represented in linear
algebra. Note, it is not the identity matrix which is represented in
Ricci notation as $\delta^i_j$. Hence, in order to represent the
derivative in a mathematical correct way, upper and lower indices are
necessary. This problem has its roots in mathematical tensor analysis,
where tensors are used for representing multilinear functions by their
values on a set of basis vectors. These values are stored in a tensor,
that is, a multi-dimensional array. Upper and lower indices are used
to distinguish between vector space and dual vector space components
that transform differently under basis changes. In the example
expression, $x$ is a vector while $x^\top$ is a co-vector from the
dual vector space.

In machine learning tensors are typically not used for representing
multi-linear functions, but simply as multi-dimensional arrays for
storing data and parameters. Hence, there is no need to distinguish
different types of components. Indices can just be used for accessing
the different tensor components. This is basically Einstein notation
that is used in all deep learning frameworks.

The contribution of this paper is an efficient and coherent method for
computing tensor derivatives in Einstein notation together with a
correctness proof. In reverse mode automatic differentiation, our
method is equivalent to the efficient approach in~\cite{LaueMG2018}
for computing higher order derivatives. Additionally, we show that
reverse mode is not optimal. A combination of reverse and forward
mode, known as cross-country mode, can be more efficient. Efficiency
can be further improved by compressing higher order derivatives.

For validating our framework we compute Hessians for several machine
learning problems. It turns out that our method, because of the
additional optimizations, outperforms the approach
of~\cite{LaueMG2018} which is already a few orders of magnitude more
efficient than TensorFlow, PyTorch, autograd, and JAX.

\paragraph{Related Work.}

Many details on the fundamentals and more advanced topics of automatic
differentiation can be found in the book
by Griewank and Walther~\cite{Griewank08}. Baydin et al.~\cite{Baydin18} provide an excellent survey on
automatic differentiation for machine learning.

Computing derivatives of non-scalar-valued functions is discussed
in the work by Pearlmutter~\cite{Pearlmutter94}. In this approach, if the function returns an
$n$-dimensional vector, then its derivative is computed by treating
each entry as a separate scalar-valued function. The same idea is
employed in almost all implementations for computing derivatives of
non-scalar-valued functions. Gebremedhin et al.~\cite{Gebremedhin09}
introduce some optimizations based on graph coloring algorithms.

Magnus and Neudecker~\cite{Magnus07} can compute derivatives with
respect to vectors and matrices. At the core of their approach,
matrices are turned into vectors by stacking columns of a matrix into
one long vector. Then, the Kronecker matrix product is used to emulate
higher order tensors. This approach works well for computing first
order derivatives of scalar-valued functions. However, it is not
practicable for computing higher order derivatives.

Giles~\cite{Giles08} collects a number of derivatives for matrix
operators, i.e., pushforward and pullback functions for automatic
differentiation. Similarly, Seeger et al.~\cite{Seeger17} provide methods and code
for computing derivatives of Cholesky factorizations, QR
decompositions, and symmetric eigenvalue decompositions. However, they
all require that the output function is scalar-valued, and hence,
cannot be generalized to higher order derivatives.

Kakade and Lee~\cite{Kakade2018} consider non-smooth functions and
provide a provably correct algorithm that returns an element from the
subdifferential. However, their algorithm is also restricted to
scalar-valued functions.

Another line of research focuses on automatic differentiation from a
programming language point of view. The goal is to incorporate
gradient computations into programming languages with the goal of
fully general differentiable programming~\cite{LeCun2018}. Work
towards this goal includes the Tangent
package~\cite{Merrienboer2018a}, the Myia
project~\cite{Merrienboer2018b}, and the approach
of Fei et al.~\cite{Fei2018}. So far this work is again restricted to
scalar-valued functions.


Recently, also second order information has been considered for
training deep nets. For instance, this information can be exploited
for escaping one of the many saddle points, but also for turning the
final classifier more robust~\cite{DeyNPP18}. Furthermore, some
differences in the convergence behavior can be explained by looking at
the spectrum of the Hessian of the objective
function~\cite{ghorbani2019,Sagun18,Yao2018}. However, so far it has
been prohibitive to compute the full Hessian even for small
networks. Here, in our experiments, we also compute the Hessian of a
small neural net.

\section{Einstein Notation}
\label{sec:tensorLanguage}
\begin{table}[ht!]
\centering    
\caption{Comparison of different linear algebra notations.}
    \label{tab:notation}
  \begin{tabular}{ccc}
      \toprule
      vectorized & Ricci & Einstein \\
      \midrule
      $y x^\top$ & $y^ix_j$ & $y*_{(i, j, ij)} x$\\\hline
      $Ax$ & $A^i_j x^j$ & $A *_{(ij, j, i)}x$\\
      $y^\top x$ & $y_ix^i$ & $y * _{(i, i, \emptyset)} x$\\
      $AB$ & $A^i_jB^j_k$  & $A *_{(ij, jk, ik)} B$\\ \hline
      $y\odot x$ & $y^ix^i $ & $y*_{(i, i, i)} x$\\
      $A\odot B$ & $A^i_jB^i_j$ & $A*_{(ij, ij, ij)} B$\\
      $A\cdot\diag(x)$ & $A^i_j x^i$ & $A*_{(ij, i, ij)} x$\\
      \bottomrule
    \end{tabular}
\end{table}

In tensor calculus one distinguishes three types of multiplication,
namely inner, outer, and element-wise multiplication. Indices are
important for distinguishing between these types. For tensors $A, B,$
and $C$ any multiplication of $A$ and $B$ can be written as
\[
C[s_3] = \sum_{(s_1\cup s_2)\setminus s_3} A[s_1] \cdot B[s_2],
\]
where $C$ is the result tensor and $s_1, s_2$, and $s_3$ are the index
sets of the left argument, the right argument, and the result tensor,
respectively. The summation is only relevant for inner products that
in Ricci calculus are denoted by shared upper and lower indices.  If
one does not want to distinguish between upper and lower indices, then
the summation must be made explicit through the result tensor. The
standard way to do so is by excluding the index for summation from the
index set of the result tensor. Hence, the index set of the result
tensor is always a subset of the union of the index sets of the
multiplication's arguments, that is, $s_3\subseteq (s_1\cup s_2)$. In
the following we denote the generic tensor multiplication simply as $C
= A *_{(s_1, s_2, s_3)} B$, where $s_3$ explicitly represents the
index set of the result tensor. This notation is basically identical
to the tensor multiplication \einsum in NumPy, TensorFlow, and
PyTorch, and to the notation used in the Tensor Comprehension
Package~\cite{Vasilache2018}.

The $*_{(s_1, s_2, s_3)}$-notation comes close to standard Einstein
notation. In Einstein notation the index set $s_3$ of the output is
omitted and the convention is to sum over all shared indices in $s_1$
and $s_2$. This, however, restricts the types of multiplications that
can be represented. The set of multiplications that can be represented
in standard Einstein notation is a proper subset of the
multiplications that can be represented by our notation. For instance,
standard Einstein notation is not capable of representing element-wise
multiplications directly. Still, in the following we refer to the
$*_{(s_1, s_2, s_3)}$-notation simply as Einstein notation as it is
standard practice in all deep learning frameworks.

Table~\ref{tab:notation} shows examples of tensor expressions in
standard linear algebra notation, Ricci calculus, and Einstein
notation. The first group shows an outer product, the second group
shows inner products, and the last group shows examples of
element-wise multiplications. As can be seen in
Table~\ref{tab:notation}, Ricci notation and Einstein notation are
syntactically reasonably similar. However, semantically they are quite
different. As pointed out above, Ricci notation differentiates between
co- and contravariant dimensions/indices and Einstein notation does
not. While this might seem like a minor difference, it does have
substantial implications when computing derivatives. For instance,
when using Ricci notation, forward and reverse mode automatic
differentiation can be treated in the same way~\cite{LaueMG2018}. This
is no longer the case when using Einstein notation.

We can show that the generic tensor multiplication operator $*_{(s_1,
  s_2, s_3)}$ is associative, commutative, and satisfies the
distributive property. Our tensor calculus, that we introduce in the
next section, makes use of all three properties. By $s_1s_2$ we denote
the concatenation of the index sets $s_1$ and~$s_2$. An example where the concatenation of two index sets is used is the outer product of two vectors, see e.g.,\ the first row in Table~\ref{tab:notation}.

\begin{lemma}[Associativity]
  \label{lem:1}
  Let $s_1, s_2, s_3$, and $s_4$ be index sets with $s_3\subseteq
  s_1\cup s_2$ and $s_4 \cap (s_1\cup s_2) = \emptyset$.  Then it
  holds that
\[
  \left(A *_{(s_1, s_2s_4, s_3s_4)} B\right) *_{(s_3s_4, s_4, s_3)} C = A *_{(s_1, s_2, s_3)} \left(B *_{(s_2s_4, s_4, s_2)} C\right).
\]
\end{lemma}

\begin{proof}
  We have
  \begin{align*}
    \left(A *_{(s_1, s_2s_4, s_3s_4)} B\right)
    *_{(s_3s_4, s_4, s_3)}  C
    &= \quad \sum_{s_4}\left(\sum_{(s_1\cup s_2)\setminus s_3}
    A[s_1] \cdot B[s_2s_4]\right)\cdot C[s_4] \\
    &= \quad\sum_{\left((s_1\cup s_2)\setminus s_3\right) \cup s_4}
    A[s_1] \cdot B[s_2s_4]\cdot C[s_4] \\
    &= \quad\sum_{(s_1\cup s_2)\setminus s_3}
    A[s_1]\left(\sum_{s_4} B[s_2s_4]\cdot C[s_4]\right) \\
    &= \quad A *_{(s_1, s_2, s_3)} \left(B *_{(s_2s_4,s_4, s_2)} C\right)
  \end{align*}
\end{proof}

Unlike standard matrix multiplication tensor multiplication is
commutative.

\begin{lemma}[Commutativity]
  \label{lem:2}
  It holds that
  \[
  A *_{(s_1, s_2, s_3)} B = B *_{(s_2, s_1, s_3)} A.
  \]
\end{lemma}
\begin{proof}
  Follows immediately from our definition of the tensor multiplication
  operator $*_{(s_1, s_2, s_3)}$, the commutativity of the scalar
  multiplication, and the com\-mutativity of the set union operation.
\end{proof}

\begin{lemma}[Distributive property]
  \label{lem:3}
  Let $s_1, s_2$, and $s_3$ be index sets with $s_3\subseteq s_1\cup
  s_2$. It holds that
  \[
  A *_{(s_1, s_2, s_3)} B + A *_{(s_1, s_2, s_3)} C = A *_{(s_1, s_2,
    s_3)} \left(B + C\right).
  \]
\end{lemma}
\begin{proof}
  Follows from the distributive property of the scalar multiplication.
\end{proof}

\section{Tensor Calculus}
\label{sec:tensorCalculus}

Now we are prepared to develop our tensor calculus. We start by giving
the definition of the derivative of a tensor-valued expression with
respect to a tensor. For the definition, we use $\norm{A}=\sqrt{\sum_s
  A[s]^2}$ as the norm of a tensor $A$ which coincides with the Euclidean norm if $A$ is a vector and with the Frobenius norm if $A$ is a matrix.

\begin{definition}[Fr\'echet Derivative] \label{def:1}
  Let $f\colon \setR^{n_1\times n_2\times\ldots\times n_k}\to
  \setR^{m_1\times m_2\times\ldots\times m_l}$ be a function that
  takes an order-$k$ tensor as input and maps it to an order-$l$
  tensor as output. Then, $D\in\setR^{m_1\times m_2\times ... \times
    m_l \times n_1\times n_2\times\ldots\times n_k}$ is called the
  derivative of $f$ at $x$ if and only if
  \[
  \lim_{h\to 0} \frac{\norm{f(x+h) - f(x) - D\circ h}}{\norm{h}} = 0,
  \]
  where $\circ$ is an inner tensor product.
\end{definition}

Here, the dot product notation $D\circ h$ is short for the inner
product $D*_{(s_1s_2, s_2, s_1)}h$, where $s_1s_2$ is the index set of
$D$ and $s_2$ is the index set of $h$. For instance, if
$D\in\setR^{m_1\times n_1 \times n_2}$ and $h\in\setR^{n_1\times
  n_2}$, then $s_1=\{i,j,k\}$ and $s_2=\{j,k\}$.

In the following, we first describe forward and reverse mode automatic
differentiation for expressions in Einstein notation, before we
discuss extensions like cross-country mode and compression of higher
order derivatives that are much easier to realize in Einstein than in
Ricci notation. As can be seen from our experiments in
Section~\ref{sec:experiments}, the extensions allow for significant
performance gains.

\subsection{Forward Mode}

Any tensor expression has an associated directed acyclic expression
graph (expression DAG). Figure~\ref{fig:dag} shows the expression DAG for the
expression
\begin{equation} \label{eq:logregexp}
X^\top (\exp(X\cdot w)+ 1)^{-1}\odot \exp(X\cdot w)),
\end{equation}
where $\odot$ denotes the element-wise multiplication and $^{-1}$ the
element-wise multiplicative inverse.
\begin{figure}[h!]
\centering
\includegraphics[height=7cm]{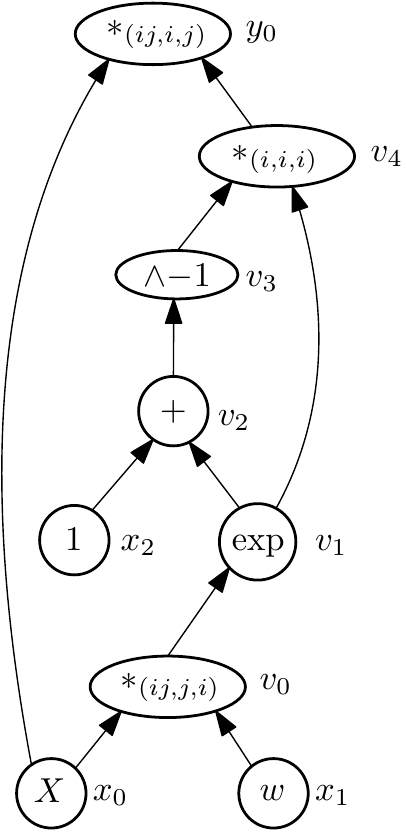}
\caption{Expression DAG for Expression~(\ref{eq:logregexp})}
\label{fig:dag}
\end{figure}
The nodes of the DAG that have no incoming edges represent the
variables of the expression and are referred to as input nodes. The
nodes of the DAG that have no outgoing edges represent the functions
that the DAG computes and are referred to as output nodes. Let the DAG
have $n$ input nodes (variables) and $m$ output nodes (functions). We
label the input nodes as $x_0, ..., x_{n-1}$, the output nodes as
$y_0, ..., y_{m-1}$, and the internal nodes as $v_0,\ldots,
v_{k-1}$. Every internal and every output node represents either a
unary or a binary operator. The arguments of these operators are
supplied by the incoming edges.

In forward mode, for computing derivatives with respect to the input
variable $x_j$, each node $v_i$ will eventually store the derivative
$\frac{\partial v_i}{\partial x_j}$ which is traditionally denoted as
$\dot v_i$. It is computed from input to output nodes as follows: At the input nodes that represent the variables
$x_i$, the derivatives $\frac{\partial x_i}{\partial x_j}$ are
stored. Hence, these are either unit tensors if $i=j$ or zero tensors otherwise.
 Then, the derivatives that are stored at the remaining nodes,
here called $f$, are iteratively computed by summing over all their
incoming edges as
\[
 \dot f = \frac{\partial f}{\partial x_j} = \sum_{z\,:\, (z, f) \in E}
 \frac{\partial f}{\partial z}\cdot \frac{\partial z}{\partial x_j} =
 \sum_{z\,:\, (z, f)\in E} \frac{\partial f}{\partial z}\cdot \dot z,
\]
where $\frac{\partial f}{\partial z}$ is the partial derivative of
node $f$ with respect to $z$ and the multiplication is tensorial. The
so called pushforwards $\dot z$ of the predecessor nodes $z$ of $f$
have been computed before and are stored at $z$. Hence, the derivative
of each function is stored at the corresponding output node $y$ of the
expression DAG. Obviously, the updates can be done simultaneously for
one input variable $x_j$ and all output nodes $y_i$. Computing the
derivatives with respect to all input variables requires $n$ such
rounds.

In the following, we derive the explicit form of the pushforward for
nodes of the expression DAG of a tensor expression. For such a DAG we
can distinguish four types of nodes, namely multiplication nodes,
general unary function nodes, element-wise unary function nodes, and
addition nodes. General unary functions are general tensor-valued
functions while element-wise unary functions are applied to each entry
of a single tensor. The difference can be best explained by the
difference between the matrix exponential function (general unary
function) and the ordinary exponential function applied to every entry
of the matrix (element-wise unary function). The pushforward for
addition nodes is trivially just the sum of the pushforward of the two
summands. Thus, it only remains to show how to compute the pushforward
for multiplication, general unary functions, and element-wise unary
function nodes. 

\begin{theorem}
  Let $x$ be an input variable with index set $s_4$ and let $C = A
  *_{(s_1, s_2, s_3)} B$ be a multiplication node of the expression
  DAG. The pushforward of $C$ is
  \[
  \dot C = B*_{(s_2, s_1s_4, s_3s_4)}\dot A + A*_{(s_1, s_2s_4,
    s_3s_4)}\dot B.
  \]
\end{theorem}
\begin{proof}  
  By the definition of the forward mode the pushforward $\dot C$ is
  given as
  \[
  \dot C = \frac{\partial C}{\partial A}\cdot \dot A + \frac{\partial
    C}{\partial B}\cdot \dot B.
  \]
  We show first how to compute $\frac{\partial C}{\partial B}\cdot
  \dot B$.  According to Definition~4 it holds that
  \[
  \lim_{h\to 0} \frac{1}{\norm{h}}\cdot \norm{B(x+h) - B(x) - \dot B \circ
    h} = 0.
  \]
  We have the following sequence of equalities
  \begin{align*}
    C(x+h) - & C(x) - \dot C \circ h \\
    & =  A *_{(s_1, s_2, s_3)} B(x+h) - A *_{(s_1, s_2, s_3)} B(x)
    - \left(A *_{(s_1, s_2s_4, s_3s_4)} \dot B\right) \circ h \\ 
    & =  A *_{(s_1, s_2, s_3)} B(x+h) - A *_{(s_1, s_2, s_3)} B(x)
    - \left(A *_{(s_1, s_2s_4, s_3s_4)} \dot B\right) *_{(s_3s_4, s_4, s_3)} h \\ 
    & =  A *_{(s_1, s_2, s_3)} B(x+h) - A *_{(s_1, s_2, s_3)} B(x)
    - A *_{(s_1, s_2, s_3)} \left(\dot B *_{(s_2s_4, s_4, s_2)} h\right) \\ 
    & =  A *_{(s_1, s_2, s_3)} \left(B(x+h) - B(x) -
    \dot B *_{(s_2s_4, s_4, s_2)} h\right) \\
    & =  A *_{(s_1, s_2, s_3)} \left(B(x+h) - B(x) - \dot B \circ h\right).
  \end{align*}
  The first equality follows from the definition of $\dot C$, the
  second from the definition of $\circ$, the third from Lemma~1, the
  fourth from Lemma~3, and the last from the definition of $\circ$.
  Thus, we have
  \begin{align*}
    \lim_{h\to 0} \frac{1}{\norm{h}}\cdot \norm{C(x+h) - C(x)
      -\dot C \circ h} 
    & = \lim_{h\to 0} \frac{1}{\norm{h}}\cdot \norm{A*_{(s_1, s_2, s_3)}\left(B(x+h)
      - B(x) -\dot B \circ h\right)} \\
    & \leq \norm{A} \lim_{h\to 0} \frac{1}{\norm{h}}\cdot \norm{B(x+h) -
        B(x) - \dot B \circ h} \\
    & = 0
  \end{align*}
  Hence, we get $\frac{\partial C}{\partial B}\cdot \dot B =A*_{(s_1,
    s_2s_4, s_3s_4)}\dot B$. Similarly, we get that $\frac{\partial
    C}{\partial A}\cdot \dot A = B*_{(s_2, s_1s_4, s_3s_4)}\dot
  A$. Combining the two equalities finishes the proof.
\end{proof}

\begin{theorem}
  Let $x$ be an input variable with index set $s_3$, let $f$ be a
  general unary function whose domain has index set $s_1$ and whose
  range has index set $s_2$, let $A$ be a node in the expression DAG,
  and let $C= f(A)$. The pushforward of the node $C$ is $\dot C =
  f^\prime(A) *_{(s_2s_1, s_1s_3, s_2s_3)} \dot A$, where $f^\prime$
  is the derivative of $f$.  \label{thm:1}
\end{theorem}
\begin{proof}
By Definition~4 we have
  \[
  \lim_{\tilde h\to 0} \frac{1}{\norm{\tilde h}}\cdot \norm{f(A+\tilde h) - f(A) - f'(A) \circ \tilde h}
  = 0.
  \]
Let $\tilde h = A(x+h)-A(x)$. Since $A$ is differentiable, we have that $\tilde h\to 0$ as $h\to 0$. Furthermore, we have that $\norm{A(x+h)-A(x)} \leq 1/c\norm{h}$ for some suitable constant $c$. Hence, we get
\begin{align}
  0 & = \lim_{h\to 0} \frac{1}{\norm{A(x+h)-A(x)}}\cdot \norm{f(A(x+h)) - f(A) - f'(A) \circ \left(A(x+h)-A(x)\right)} \nonumber \\
 & \geq \lim_{h\to 0} \frac{c}{\norm{h}}\cdot \norm{f(A(x+h)) - f(A) - f'(A) \circ \left(A(x+h)-A(x)\right)} \label{eq:t9}
\end{align}

By Definition~4 we also have that
  \[
  \lim_{h\to 0} \frac{1}{\norm{h}}\cdot \norm{A(x+h) - A(x) - \dot A \circ h}
  = 0.
  \]

Hence, we can replace in the limit $A(x+h)-A(x)$ with $\dot A\circ h$ in \eqref{eq:t9} and obtain
\[
  0 \geq \lim_{h\to 0} \frac{c}{\norm{h}}\cdot \norm{f(A(x+h)) - f(A) - f'(A) \circ (\dot A \circ h)}.
\]
Note, that 
\begin{align*}
f'(A) \circ (\dot A \circ h) & = f'(A(x)) \circ (\dot A *_{(s_1s_3, s_3, s_1)} h) \\
    &= f'(A(x)) *_{(s_2s_1, s_1, s_2)}  (\dot A *_{(s_1s_3, s_3, s_1)} h) \\
    &= (f'(A(x)) *_{(s_2s_1, s_1s_3, s_2s_3)}  \dot A) *_{(s_2s_3, s_3, s_2)} h)\\
    &= (f'(A(x)) *_{(s_2s_1, s_1s_3, s_2s_3)}  \dot A) \circ h)
\end{align*}
Hence, we obtain
\[
  0 \geq \lim_{h\to 0} \frac{c}{\norm{h}}\cdot \norm{f(A(x+h)) - f(A) - (f'(A(x)) *_{(s_2s_1, s_1s_3, s_2s_3)}  \dot A) \circ h)}.
\]
Thus, we get $\dot C = f'(A) *_{(s_2s_1, s_1s_3, s_2s_3)} \dot A$ as claimed.
\end{proof}

In case that the general unary function is simply an element-wise unary
function that is applied element-wise to a tensor, Theorem~\ref{thm:1}
simplifies as follows.

\begin{theorem}
  Let $x$ be an input variable with index set $s_2$, let $f$ be an
  element-wise unary function, let $A$ be a node in the expression DAG
  with index set $s_1$, and let $C= f(A)$ where $f$ is applied
  element-wise. The pushforward of the node $C$ is $\dot C =
  f^\prime(A) *_{(s_1, s_1s_2, s_1s_2)} \dot A$, where $f^\prime$ is
  the derivative of $f$.
\end{theorem}
\begin{proof}
  By Definition~4 we have
  \[
  \lim_{h\to 0} \frac{1}{\norm{h}}\cdot \norm{A(x+h) - A(x) - \dot A \circ h}
  = 0.
  \]
  It follows that for every scalar tensor entry $A(x)_s$, where $s$ is
  the multi-index of the entry, that
  \[
  \lim_{h\to 0} \frac{1}{\norm{h}}\cdot \big|A(x+h)_s - A(x)_s - (\dot A
  \circ h)_s\big| = 0.
  \]
  Let $f^h(A,x) = f(A(x+h) - f(A(x))$. Since, $f$ is applied entrywise
  and $f'$ is the derivative of $f$ we get from the chain rule for the
  scalar case that
  \[
    \lim_{h\to 0} \frac{1}{\norm{h}}\cdot \big|f^h(A,x)_s - f'(A(x)_s)
    \cdot (\dot A \circ h)_s\big| = 0.
  \]
  Since this equality holds for all multi-indices $s$ we get by
  summing over these indices that
  \[
  \lim_{h\to 0} \frac{1}{\norm{h}}\cdot \norm{f^h(A,x) - f'(A(x)) *_{(s_1,
      s_1, s_1)} (\dot A \circ h)} = 0.
  \]
  We have
  \begin{align*}
    f'(A(x)) *_{(s_1, s_1, s_1)}  (\dot A \circ h)
    &= f'(A(x)) *_{(s_1, s_1, s_1)}  (\dot A *_{(s_1s_2, s_2, s_1)} h) \\
    &= (f'(A(x)) *_{(s_1, s_1s_2, s_1s_2)}  \dot A) *_{(s_1s_2, s_2, s_1)} h\\
    &= (f'(A(x)) *_{(s_1, s_1s_2, s_1s_2)}  \dot A) \circ h,
  \end{align*}
  where the first and the last equality follow from the definition of
  $\circ$, and the second equality follows from Lemma~1. Hence, we
  have
  \[
    \lim_{h\to 0} \bigg(\frac{1}{\norm{h}}\cdot \bnorm
    f(A(x+h))-f(A(x)) - (f'(A(x)) *_{(s_1, s_1s_2, s_1s_2)} \dot
    A)\circ h)\bnorm \bigg) = 0.
  \]
  Thus, we get $\dot C = f'(A) *_{(s_1, s_1s_2, s_1s_2)} \dot A$.
\end{proof}

\subsection{Reverse Mode}

Reverse mode automatic differentiation proceeds similarly to the
forward mode, but from output to input nodes. Each node $v_i$ will
eventually store the derivative $\frac{\partial y_j}{\partial v_i}$
which is usually denoted as $\bar v_i$, where $y_j$ is the function to
be differentiated. These derivatives are computed as follows: First,
the derivatives $\frac{\partial y_j}{\partial y_i}$ are stored at the
output nodes of the DAG. Hence again, these are either unit tensors if $i=j$ or zero tensors otherwise. Then, the derivatives that are stored at the
remaining nodes, here called $z$, are iteratively computed by summing
over all their outgoing edges as follows
\[
 \bar z = \frac{\partial y_j}{\partial z} = \sum_{f\,:\, (z, f)\in E}
 \frac{\partial y_j}{\partial f}\cdot \frac{\partial f}{\partial z} =
 \sum_{f\,:\, (z, f)\in E} \bar f \cdot \frac{\partial f}{\partial z},
\]
where the multiplication is again tensorial. The so-called pullbacks
$\bar f$ have been computed before and are stored at the successor
nodes $f$ of $z$. This means the derivatives of the function $y_j$
with respect to all the variables $x_i$ are stored at the
corresponding input nodes of the expression DAG. Computing the
derivatives for all the output functions requires $m$ such rounds.

In the following we describe the contribution of unary and binary
operator nodes to the pullback of their arguments.
We have only two types of binary operators, namely tensor
addition and tensor multiplication. In the addition case the
contribution of $C$ to the pullback of both of its arguments is simply
$\bar C$. In Theorem~\ref{thm:pullbackMult} we derive the explicit
form of the contribution of a multiplication node to the pullback of
its arguments, in Theorem~\ref{thm:2} the contribution of a general unary
function, and in Theorem~\ref{thm:pullbackUnary} we derive the
contribution of an element-wise unary function node to its argument.

\begin{theorem} \label{thm:pullbackMult}
  Let $Y$ be an output node with index set $s_4$ and let $C = A
  *_{(s_1, s_2, s_3)} B$ be a multiplication node of the expression
  DAG. Then the contribution of $C$ to the pullback $\bar B$ of $B$
  is $\bar C*_{(s_4s_3, s_1, s_4s_2)}A$ and its contribution to the
  pullback $\bar A$ of $A$ is $\bar C*_{(s_4s_3, s_2, s_4s_1)}B$.
\end{theorem}
\begin{proof}
  Here we only derive the contribution of $C$ to the pullback $\bar
  B$. Its contribution to $\bar A$ can be computed analogously. The
  contribution of $C$ to $\bar B$ is $\bar C\cdot\frac{\partial
    C}{\partial B}$. By Definition~\ref{def:1} we have for the
  derivative $\bar C=\frac{\partial Y}{\partial C}$ of $Y$ with
  respect to $C$ that
  \[
    \lim_{\tilde h\to 0}\, \frac{1}{\|\tilde h\|}\cdot \Big\|
    Y(C+\tilde h) - Y(C) - \bar C \circ \tilde h\Big\| = 0.
  \] 
  By specializing $\tilde h = A *_{(s_1, s_2, s_3)} h$ we get
  \begin{align*}
     Y(C+&\tilde h) - Y(C) - \bar C \circ \tilde h \\
    & =  Y(A *_{(s_1, s_2, s_3)} B + A *_{(s_1, s_2, s_3)} h) - Y(A *_{(s_1, s_2, s_3)} B) - \bar C\circ (A*_{(s_1, s_2, s_3)} h) \\ 
    & =  Y(A *_{(s_1, s_2, s_3)} (B + h)) - Y(A *_{(s_1, s_2, s_3)} B) - \bar C\circ (A*_{(s_1, s_2, s_3)} h) \\ 
    & =  Y(A *_{(s_1, s_2, s_3)} (B + h)) - Y(A *_{(s_1, s_2, s_3)} B) - \bar C *_{(s_4s_3, s_3, s_4)} (A*_{(s_1, s_2, s_3)} h) \\ 
    & =  Y(A *_{(s_1, s_2, s_3)} (B + h)) - Y(A *_{(s_1, s_2, s_3)} B) - (\bar C *_{(s_4s_3, s_1, s_4s_2)} A)*_{(s_4s_2, s_2, s_4)} h) \\ 
    & =  Y(A *_{(s_1, s_2, s_3)} (B + h)) - Y(A *_{(s_1, s_2, s_3)} B)  - (\bar C *_{(s_4s_3, s_1, s_4s_2)} A)\circ h), 
  \end{align*}
  where the first equality follows from the definitions of $C$ and
  $\tilde h$, the second from Lemma~\ref{lem:3}, the third from the
  definition of~$\circ$, the fourth from Lemma~\ref{lem:1}, the fifth
  from Lemma~\ref{lem:2}, and the last again from the definition
  of~$\circ$.  Hence, we have for $\frac{\partial Y}{\partial
    C}\cdot\frac{\partial C}{\partial B}$ that
     \begin{align*}
    0 & = \lim_{\tilde h\to 0} \frac{1}{\big\|\tilde h\big\|}\cdot
    \norm{Y(C+\tilde h) - Y(C) - \bar C \circ \tilde h} \\
    & = \lim_{h\to 0} \frac{1}{\norm{h}}\cdot \bnorm
    Y(A *_{(s_1, s_2, s_3)} (B + h)) - Y(A *_{(s_1, s_2, s_3)} B) -
    (\bar C *_{(s_4s_3, s_1, s_4s_2)} A)\circ h)\bnorm  
  \end{align*}
  Thus, the contribution of $C$ to the pullback $\bar B$ is
  \[
  \frac{\partial Y}{\partial C}\cdot\frac{\partial C}{\partial B} =
  \bar C\cdot \frac{\partial C}{\partial B} = \bar C *_{(s_4s_3, s_1,
    s_4s_2)} A.
  \]
\end{proof}

If the output function $Y$ in Theorem~\ref{thm:pullbackMult} is
scalar-valued, then we have $s_4=\emptyset$ and the pullback function
coincides with the function implemented in all modern deep learning
frameworks including TensorFlow and PyTorch. Hence, our approach can
be seen as a direct generalization of the scalar case.

\begin{theorem} \label{thm:2}
  Let $Y$ be an output function with index set $s_3$, let $f$ be a
  general unary function whose domain has index set $s_1$ and whose range has
  index set $s_2$, let $A$ be a node in the expression DAG, and let
  $C= f(A)$. The contribution of the node $C$ to the pullback $\bar A$
  is \[ \bar f*_{(s_3s_2,s_2s_1,s_3s_1)} f'(A), \] where $f^\prime$ is
  the derivative of $f$.
\end{theorem}
\begin{proof}
  The contribution of the node $C$ to the pullback $\bar A$ is $\bar
  f\cdot\frac{\partial f}{\partial A}$.  By Definition~4 we have for
  the derivative $\bar f=\frac{\partial Y}{\partial f}$ of $Y$ with
  respect to $f$ that
  \begin{equation} \label{eq:t31}
    \lim_{\tilde h\to 0} \frac{1}{\big\|\tilde h\big\|}\cdot
    \norm{Y(f+\tilde h) - Y(f) - \bar f\circ \tilde h} = 0.
  \end{equation}
  By specializing $\tilde h = f(A+h) - f(A)$ and setting $f = f(A)$ we
  get
  \begin{align*}
    & Y(f+\tilde h) - Y(f) - \bar f \circ \tilde h \\
    & \qquad\qquad\qquad=\, Y(f(A+h) -f(A) + f(A)) - Y(f(A))
    - \bar f(A) \circ (f(A+h)-f(A)) \\
    & \qquad\qquad\qquad=\, Y(f(A+h)) - Y(f(A))
    - \bar f(A) \circ (f(A+h)-f(A))
  \end{align*}
  Furthermore, since $f'$ is the derivative of $f$ we have
  \begin{equation} \label{eq:t32}
    \lim_{h\to 0} \frac{1}{\norm{h}}\cdot \norm{\tilde h -f'(A)
        *_{(s_2s_1,s_1,s_2)}h} = 0.
  \end{equation}
  Combining Equations~\eqref{eq:t31} and~\eqref{eq:t32} gives
  \begin{align*}
    0 & = \lim_{\tilde h\to 0} \frac{1}{{\lnorm{\tilde h}}}\cdot
    \norm{Y(f+\tilde h) - Y(f) - \bar f \circ \tilde h} \\
    & = \lim_{h\to 0} \frac{1}{\norm{h}}\cdot \bnorm Y(f(A+h)) - Y(f(A)) 
    - \bar f(A) \circ (f'(A) *_{(s_2s_1,s_1,s_2)} h)\bnorm \\
    & = \lim_{h\to 0} \frac{1}{\norm{h}}\cdot \bnorm Y(f(A+h)) - Y(f(A)) 
    - \bar f(A) *_{(s_3s_2,s_2,s_3)}
    (f'(A) *_{(s_2s_1,s_1,s_2)} h)\bnorm  \\
    & = \lim_{h\to 0} \frac{1}{\norm{h}}\cdot \bnorm Y(f(A+h)) - Y(f(A))
    - (\bar f(A) *_{(s_3s_2,s_2s_1,s_3s_1)}
    f'(A)) *_{(s_3s_1,s_1,s_3)} h\bnorm \\ 
    & = \lim_{h\to 0} \frac{1}{\norm{h}}\cdot \bnorm Y(f(A+h)) - Y(f(A)) 
    - (\bar f(A) *_{(s_3s_2,s_2s_1,s_3s_1)}
    f'(A)) \circ h\bnorm  \\
  \end{align*}
  Hence, the contribution of the node $C$ to the pullback $\bar A$ is
  \[
  \frac{\partial Y}{\partial f}\cdot\frac{\partial f}{\partial A} =
  \bar f\cdot\frac{\partial f}{\partial A} = \bar
  f*_{(s_3s_2,s_2s_1,s_3s_1)} f'(A).
  \]
\end{proof}

In case that the general unary function is simply an element-wise unary
function that is applied element-wise to a tensor, Theorem~\ref{thm:2}
simplifies as follows.

\begin{theorem} \label{thm:pullbackUnary}
  Let $Y$ be an output function with index set $s_2$, let $f$ be an
  element-wise unary function, let $A$ be a node in the expression DAG
  with index set $s_1$, and let $C= f(A)$ where $f$ where $f$ is
  applied element-wise. The contribution of the node $C$ to the
  pullback $\bar A$ is \[ \bar f*_{(s_2s_1,s_1,s_2s_1)} f'(A), \]
  where $f^\prime$ is the derivative of $f$.
\end{theorem}
\begin{proof}
  The contribution of the node $C$ to the pullback $\bar A$ is $\bar
  f\cdot\frac{\partial f}{\partial A}$.  By Definition~4 we have for
  the derivative $\bar f=\frac{\partial Y}{\partial f}$ of $Y$ with
  respect to $f$ that
  \begin{equation} \label{eq:31}
    \lim_{\tilde h\to 0} \frac{1}{\big\|\tilde h\big\|}\cdot
    \norm{Y(f+\tilde h) - Y(f) - \bar f\circ \tilde h} = 0.
  \end{equation}
  By specializing $\tilde h = f(A+h) - f(A)$ and setting $f = f(A)$ we
  get
  \begin{align*}
    & Y(f+\tilde h) - Y(f) - \bar f \circ \tilde h \\
    & \qquad\qquad\qquad=\, Y(f(A+h) -f(A) + f(A)) - Y(f(A))
    - \bar f(A) \circ (f(A+h)-f(A)) \\
    & \qquad\qquad\qquad=\, Y(f(A+h)) - Y(f(A))
    - \bar f(A) \circ (f(A+h)-f(A))
  \end{align*}
  Furthermore, since $f'$ is the derivative of $f$ and $f$ is an
  entrywise function we have
  \begin{equation} \label{eq:32}
    \lim_{h\to 0} \frac{1}{\norm{h}}\cdot \norm{\tilde h -f'(A)
        *_{(s_1,s_1,s_1)}h} = 0.
  \end{equation}
  Combining Equations~\eqref{eq:31} and~\eqref{eq:32} gives
  \begin{align*}
    0 & = \lim_{\tilde h\to 0} \frac{1}{{\lnorm{\tilde h}}}\cdot
    \norm{Y(f+\tilde h) - Y(f) - \bar f \circ \tilde h} \\
    & = \lim_{h\to 0} \frac{1}{\norm{h}}\cdot \bnorm Y(f(A+h)) - Y(f(A)) 
    - \bar f(A) \circ (f'(A) *_{(s_1,s_1,s_1)} h)\bnorm \\
    & = \lim_{h\to 0} \frac{1}{\norm{h}}\cdot \bnorm Y(f(A+h)) - Y(f(A)) 
    - \bar f(A) *_{(s_2s_1,s_1,s_2)}
    (f'(A) *_{(s_1,s_1,s_1)} h)\bnorm  \\
    & = \lim_{h\to 0} \frac{1}{\norm{h}}\cdot \bnorm Y(f(A+h)) - Y(f(A))
    - (\bar f(A) *_{(s_2s_1,s_1,s_2s_1)}
    f'(A)) *_{(s_2s_1,s_1,s_2)} h\bnorm \\ 
    & = \lim_{h\to 0} \frac{1}{\norm{h}}\cdot \bnorm Y(f(A+h)) - Y(f(A)) 
    - (\bar f(A) *_{(s_2s_1,s_1,s_2s_1)}
    f'(A)) \circ h\bnorm  \\
  \end{align*}
  Hence, the contribution of the node $C$ to the pullback $\bar A$ is
  \[
  \frac{\partial Y}{\partial f}\cdot\frac{\partial f}{\partial A} =
  \bar f\cdot\frac{\partial f}{\partial A} = \bar
  f*_{(s_2s_1,s_1,s_2s_1)} f'(A).
  \]
\end{proof}

\subsection{Beyond Forward and Reverse Mode}
\label{sec:improve}

Since the derivative of a function $y$ with respect to an input
variable $x$ is the sum over all partial derivatives along all paths
from $x$ to $y$, see e.g.,~\cite{Griewank08}, we can combine forward and reverse
mode. Using that $\bar v = \frac{\partial y}{\partial v}$ and $\dot v =
\frac{\partial v}{\partial x}$, we get
\[
\frac{\partial y}{\partial x} = \sum_{v\in S} \bar v *_{(s_1s_v,
  s_vs_2, s_1s_2)}\dot v,
\]
where $s_v$ is the index set of node $v$, $s_1$ is the index set of
the output function $y$, $s_2$ is the index set of the input node $x$,
and $S$ is the set of nodes in a cut of the expression DAG.  General
combinations of forward and reverse mode lead to the so-called
cross-country mode. We will show that the differentiation of tensor
expressions becomes even more efficient by a special instantiation of
the cross-country mode and by compressing higher order derivatives.

\paragraph{Cross-Country Mode.}

In both forward and reverse mode, derivatives are computed as
sums of products of partial derivatives. In general, the time for
evaluating the derivatives depends on the order by which the partial
derivatives are multiplied. The two modes multiply the partial
derivatives in opposite order. Derivatives are multiplied from input
to output nodes in forward mode and vice versa in reverse
mode.

If the output function is scalar-valued, then reverse mode is
efficient for computing the derivative with respect to all input
variables.  It is guaranteed that evaluating the derivative takes at
most six times the time for evaluating the function itself. In
practice, usually a factor of two is
observed~\cite{Griewank08}. However, this is no longer true for
non-scalar-valued functions. In the latter case, the order of
multiplying the partial derivatives has a strong impact on the
evaluation time, even for simple
functions, see e.g., Naumann~\cite{Naumann04}. Reordering the multiplication order of the
partial derivatives is known as cross-country mode in the automatic
differentiation literature~\cite{Bischof02}. Finding an optimal
ordering is NP-hard~\cite{Naumann08} in general.

However, it turns out that significant performance gains for
derivatives of tensor expressions can be obtained by the re-ordering
strategy that multiplies tensors in order of their tensor-order, that
is, multiplying vectors first, then matrices, and so on. We illustrate
this strategy on the following example
\begin{equation}\label{eq:XC}
f(x) = B\cdot g(h(Ax)),
\end{equation}
where $A$ and $B$ are two matrices, $x$ is a vector and $g(.)$ and
$h(.)$ are vector-valued functions that also take a vector as input.
The derivative in this case is $B\diag(u)\diag(v)A$, where $u=g^\prime
(h(Ax)),\, v=h^\prime (Ax)$, and $\diag(u)$ is the diagonal matrix
with $u$ on its diagonal. Reverse mode multiplies these matrices from
left to right while forward mode multiplies them from right to
left. However, it is more efficient to first multiply the two vectors
$u$ and $v$ element-wise and then to multiply the result with the
matrices $A$ and $B$.

Actually, the structure of Example~\ref{eq:XC} is not contrived, but
fairly common in second order derivatives. For instance, consider the
expression $\sum g(h(Ax))$, where $g$ and $h$ are as above and the sum
is over the vector components of the vector-valued expression
$g(h(Ax))$. Many machine learning problems feature such an expression
as subexpression, where $A$ is a data matrix and the optimization
variable $x$ is a parameter vector. The gradient of this expression
has the form of Example~\ref{eq:XC} with $B=A^\top$. As can be seen in
the experiments in Section~\ref{sec:experiments}, reordering the
multiplications by our strategy reduces the time for evaluating the
Hessian by about $30\%$.

\paragraph{Compressing Derivatives.} Our compression scheme builds
on the re-ordering scheme (cross-country mode) from above and on the
simple observation that in forward as well as in reverse mode the
first partial derivative is always a unit tensor. It is either, in
reverse mode, the derivative of the output nodes with respect to
themselves or, in forward mode, the derivative of the input nodes with
respect to themselves. This unit tensor can always be moved to the end
of the multiplications, if the order of multiplication is chosen
exactly as in our cross-country mode strategy that orders the tensors
in increasing tensor-order. Then, the multiplication with the unit
tensor at the end is either trivial, i.e., amounts to a multiplication
with a unit matrix that has no effect and thus can be removed, or
leads to a compactification of the derivative.

For an example, consider the loss function
\[
f(U) = \norm{T - UV^\top}^2
\]
of the non-regularized matrix factorization problem which is often used for recommender systems~\cite{Koren09}. Here,
$T\in\setR^{n\times n},\, U, V\in\setR^{n\times k}$ and $n$ is usually
large while $k$ is small. The Hessian of $f$ is the fourth order
tensor
\[
H = 2(V*_{(ij, ik, jk)}V) *_{(jl, ik, ijkl)} \id
\in\setR^{n\times k\times n\times k},
\]
where $\id$ is the identity
matrix. Newton-type algorithms for this problem solve the Newton
system which takes time in $O\left((nk)^3\right)$. However, the
Hessian can be compressed to $2(V*_{(ij, ik, jk)}V)$ which is a small
matrix of size $k\times k$. This matrix can be inverted in $O(k^3)$
time. The performance gain realized by compression can be significant.
For instance, solving the compressed Newton system needs only about
$10~\mu$sec whereas solving the original system needs about $1$~sec
for a problem of size $n=1000$ and $k=10$. For more experimental
results please refer to Section~\ref{sec:experiments}.

As another example, consider a simple neural net with a fixed number
of fully connected layers, ReLU activation functions, and a softmax
cross-entropy output layer.  The Hessian of each layer is a fourth
order tensor that can be written as $A*_{(ijl, ik, ijkl)} \id$ for a
suitable third order tensor $A$. In this case, the Hessian can be
compressed from a fourth order tensor to a third order tensor. For illustrative purposes, we
provide expression trees for both derivatives, compressed and
uncompressed, in the appendix. Computing with the compact
representation of the Hessian is of course more efficient which we
confirm experimentally in the next section.


\section{Experiments}
\label{sec:experiments}

We have implemented both modes of the tensor calculus from the
previous section together with the improvements that can be achieved
by cross-country mode and the compactification of higher order
derivatives. State-of-the-art frameworks like TensorFlow and PyTorch
only support reverse mode since it allows to compute derivatives with
respect to all input variables at the same time. Similarly to all
other frameworks, our implementation performs some expression
simplification like constant folding and removal of zero and identity
tensors. 

\begin{figure*}[h!]
\begin{center}
  \includegraphics[width=0.32\textwidth]{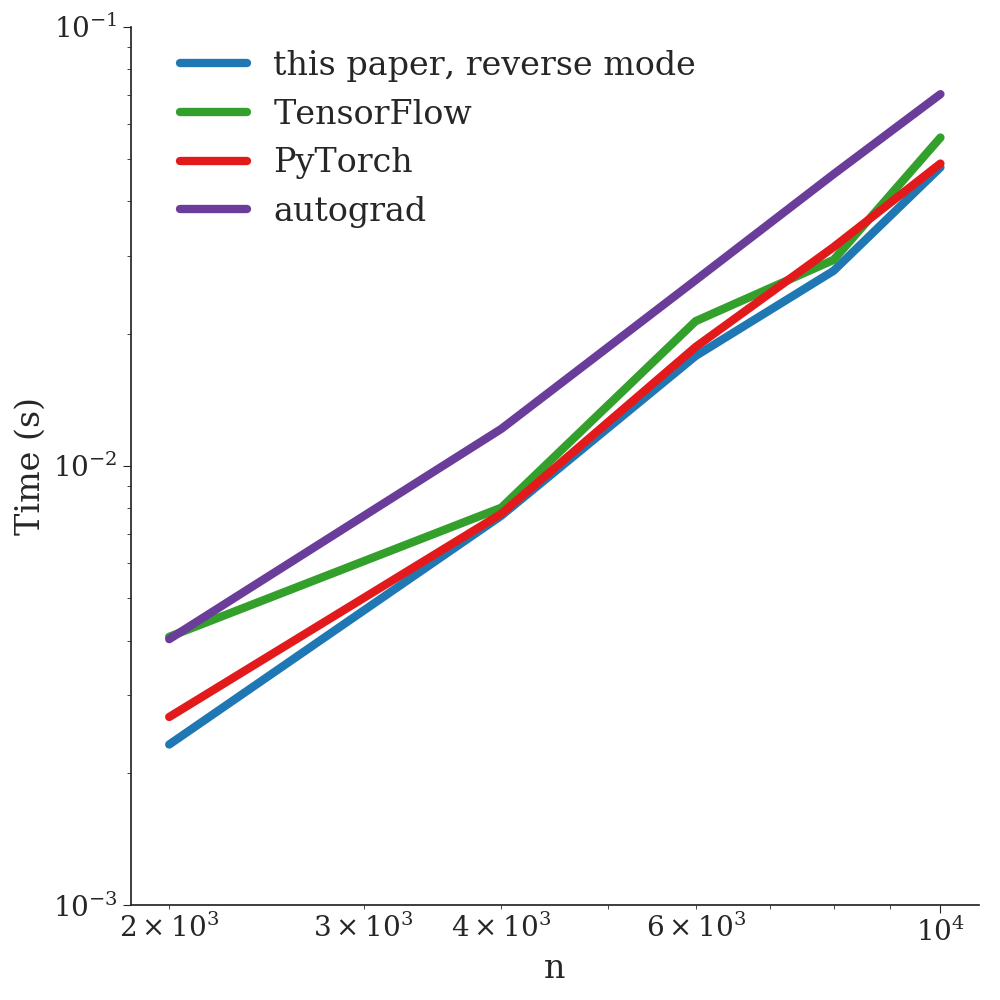}
  \includegraphics[width=0.32\textwidth]{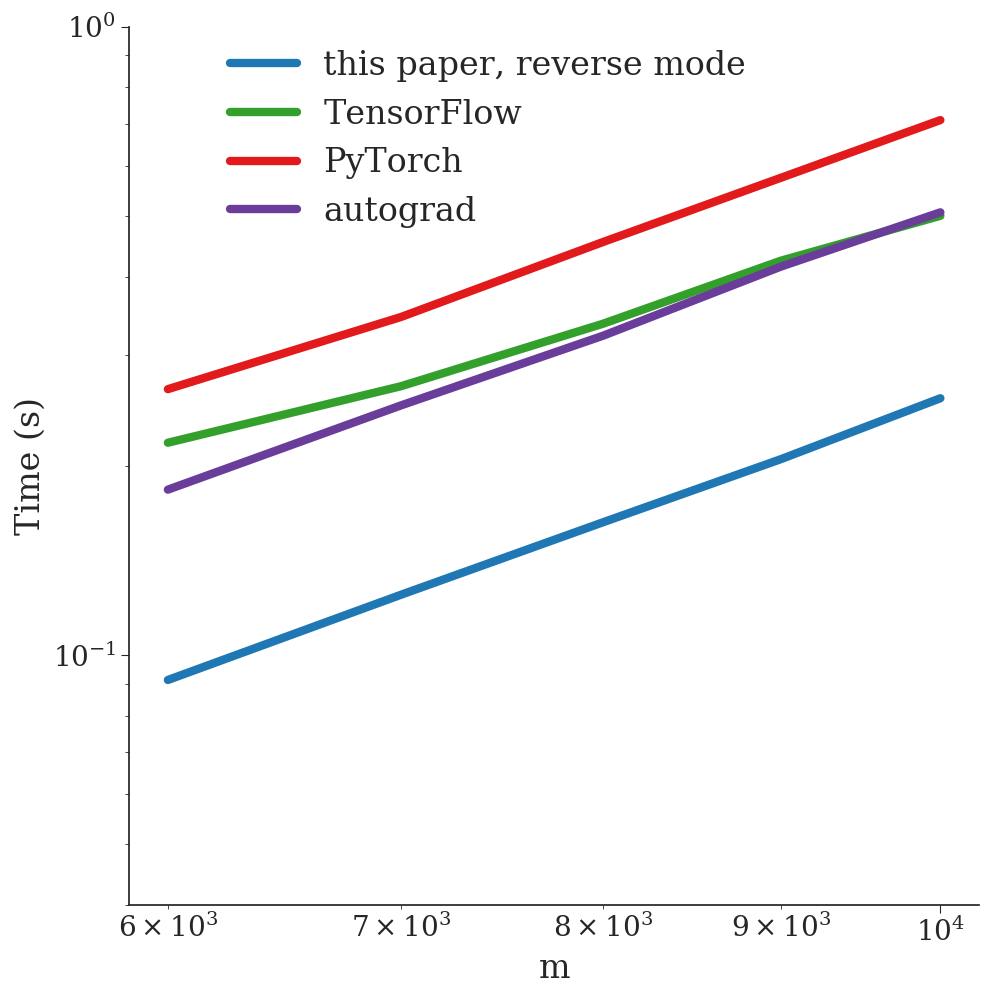}
  \includegraphics[width=0.32\textwidth]{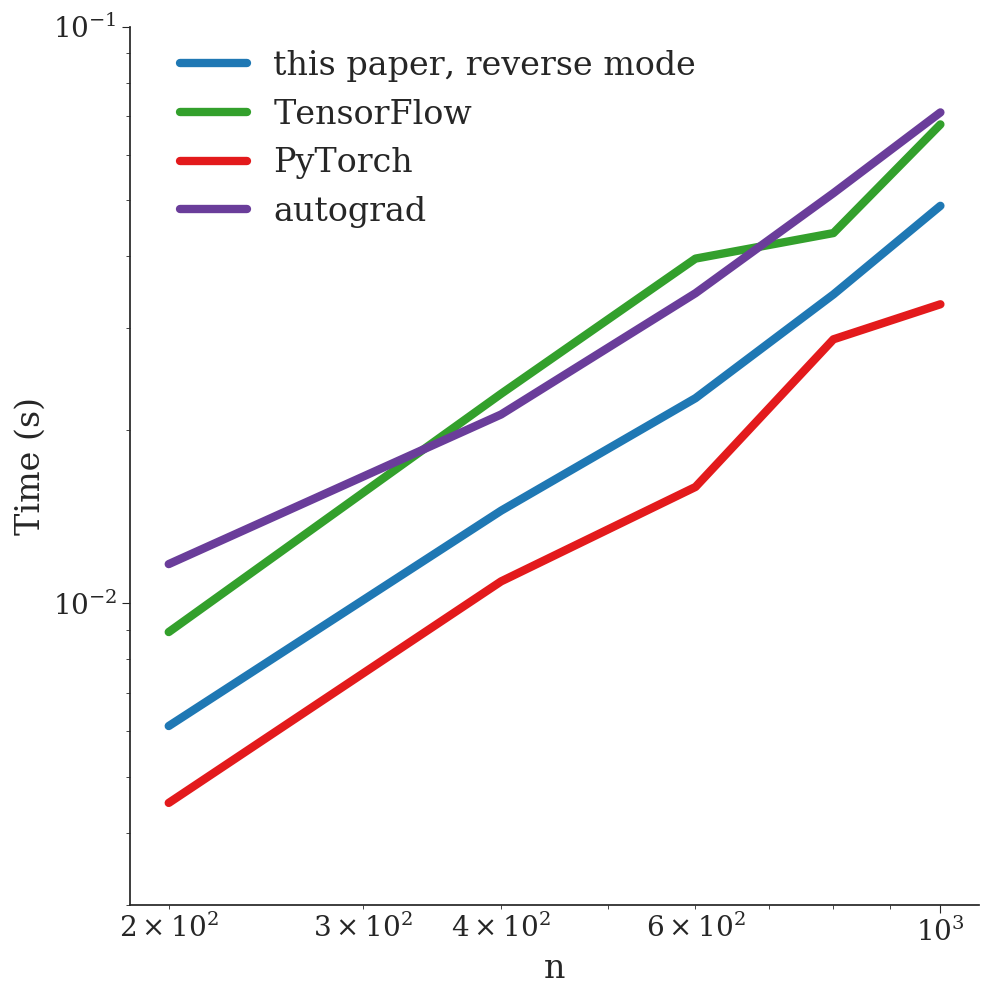}
  \caption{Running times for computing function value and gradient for
    the logistic regression function (left), matrix factorization
    (middle), and a small neural net (right). The times were measured
    on a CPU.}
  \label{fig:ex1}
\end{center}
\end{figure*}

\begin{figure*}[t!]
	\centering
    \begin{tabular}{lcr}
      \includegraphics[width=0.31\textwidth]{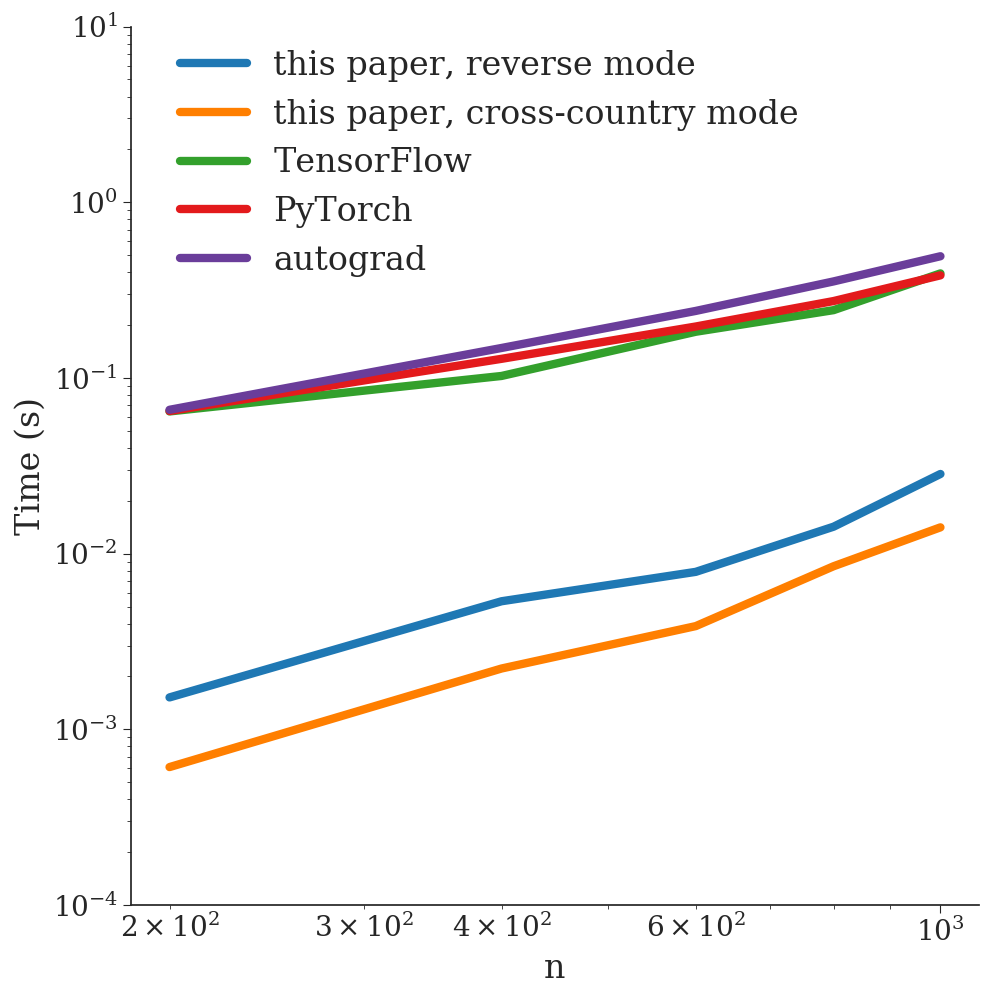} &
      \includegraphics[width=0.31\textwidth]{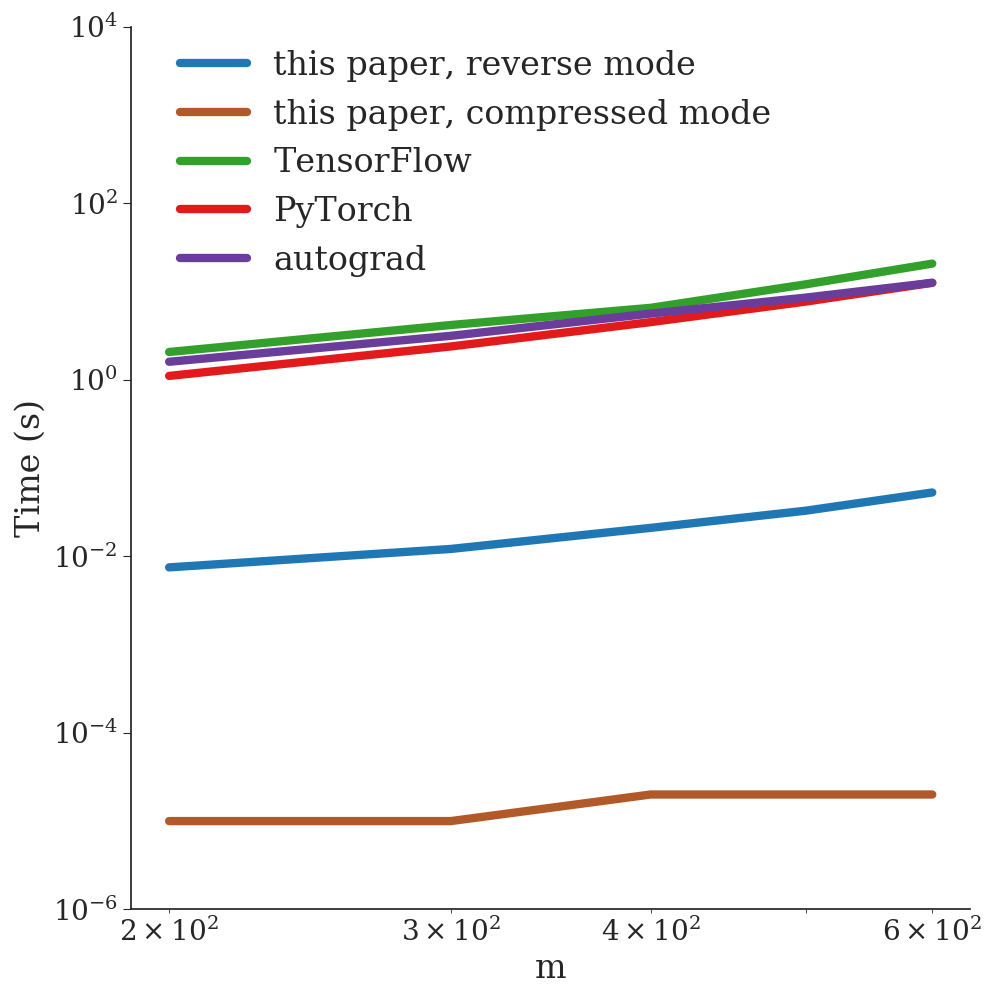} &
      \includegraphics[width=0.31\textwidth]{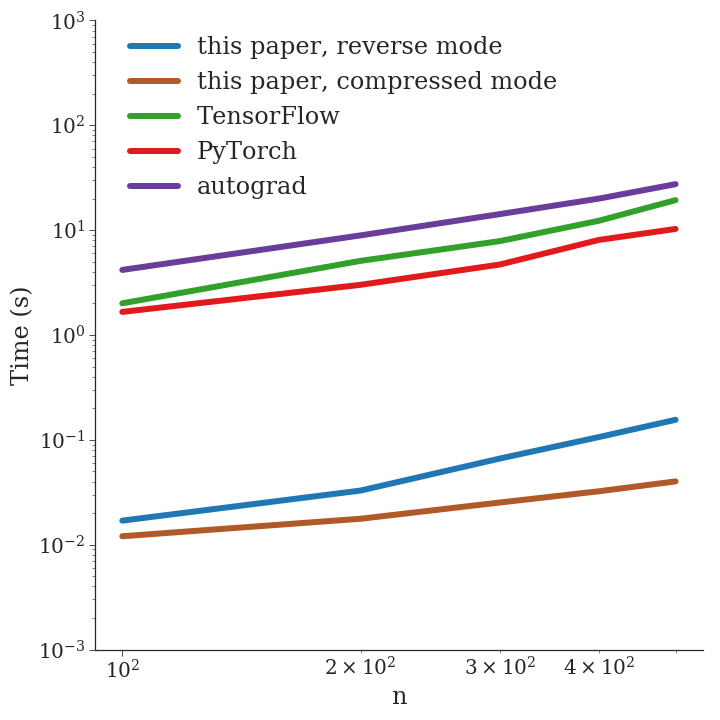} \\
      \includegraphics[width=0.31\textwidth]{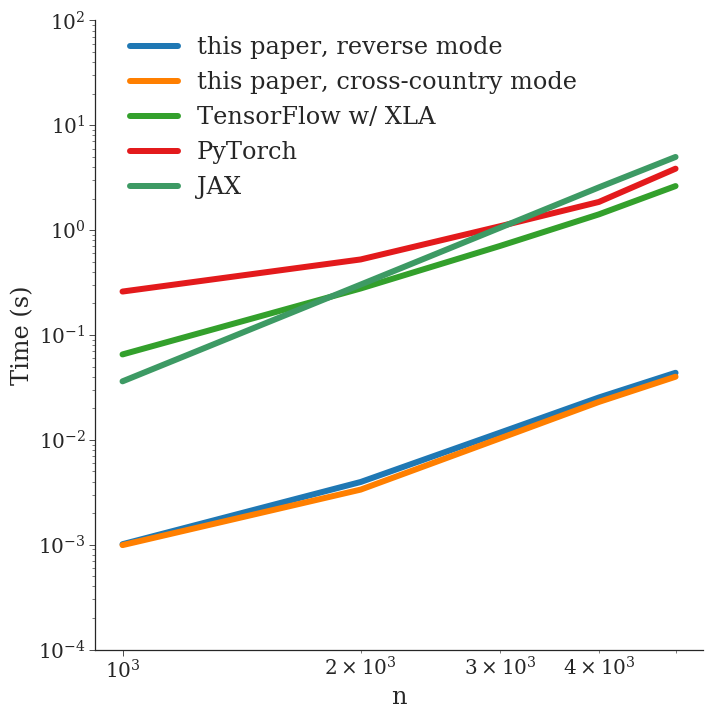} & 
      \includegraphics[width=0.31\textwidth]{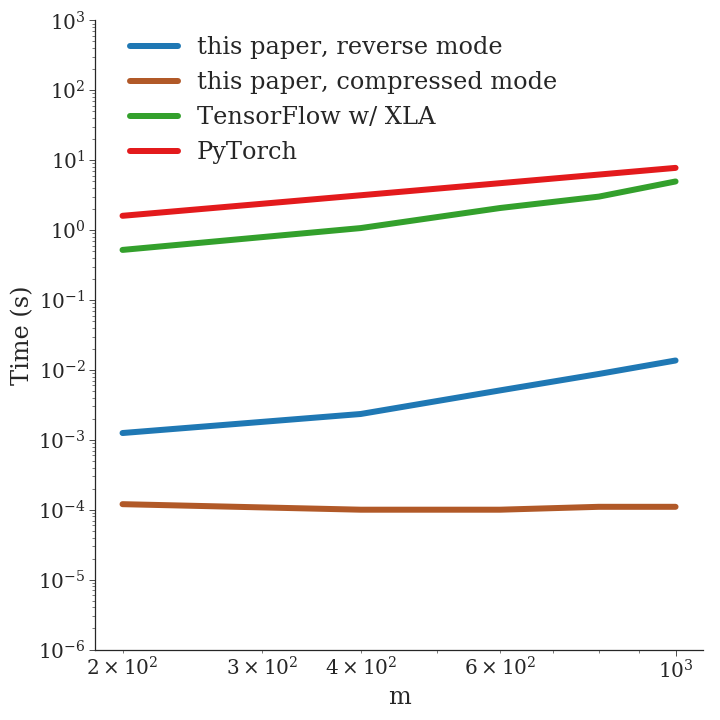} &
      \includegraphics[width=0.31\textwidth]{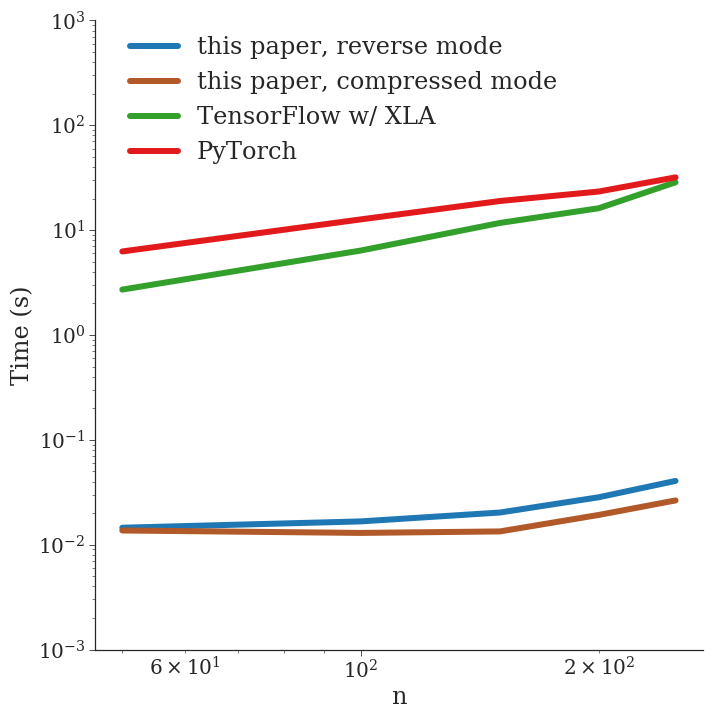}
    \end{tabular}
    \caption{Running times on a CPU (top row) and on a GPU (bottom
      row) for computing the Hessian of the logistic regression
      function (left), matrix factorization (middle), and a small
      neural net (right).}
      \label{fig:ex2}
\end{figure*}

\paragraph{Experimental Setup.} We followed the experimental setup
of Laue et al.~\cite{LaueMG2018}. However, we added one more experiment, a small
neural net.  We have implemented our algorithms in Python. To evaluate
expressions we used NumPy~1.16 and CuPy~5.1.  We compared our
framework with the state-of-the-art automatic differentiation
frameworks that natively support linear algebra operations TensorFlow~1.14, PyTorch~1.0, autograd~1.2, and JAX~0.1.27
used with Python~3.6 that were all linked against the Intel MKL. All
these frameworks support reverse mode automatic differentiation for
computing first order derivatives. For scalar-valued functions the
reverse mode of each of these frameworks coincides with the reverse
mode of our approach. For non-scalar-valued functions all the
frameworks compute the derivative for each entry of the output
function separately. The expression DAGs that are generated by our
reverse mode for general tensor expressions coincide with the
derivatives computed by the approach presented in~\cite{LaueMG2018}.  The
experiments were run in a pure CPU setting (Intel Xeon E5-2643, 8
cores) as well as in a pure GPU setting (NVIDIA Tesla V100), except
for autograd that does not provide GPU support.

We computed function values, gradients, and Hessians for each set of
experiments. We computed Hessians on the CPU as well as on the GPU. To
avoid the issue of sparsity patterns we generated dense, random data
for each experiment. In this setting the running time does not depend
on whether the data are synthetic or real world.

\paragraph{Logistic regression.}  Logistic regression~\cite{Cox58} is
probably one of the most commonly used methods for
classification. Given a set of $m$ data points $X\in\setR^{m\times n}$
along with a set of binary labels $y\in\{\pm 1\}^m$, logistic
regression aims at minimizing the loss function $\sum_i
\log\left(\exp\left(-y^{(i)}\, (X^{(i)}w)\right) + 1\right)$, where
$w\in\setR^n$ is the weight vector, $X^{(i)}$ is the $i$-th data point
($i$-th row of $X$), and $y^{(i)}$ the corresponding $i$-th label. The
data matrix $X$ can be composed of the true input features, features
transformed by basis
functions/kernels~\cite{Broomhead88,Schoelkopf02}, or by random basis
functions~\cite{RahimiR07}, or by features that have been learned by a
deep net~\cite{Hinton06}. We set $m = 2n$ in the experiments.

\paragraph{Matrix factorization.} Matrix factorization can be stated as
the problem $\min_{U, V} \|T-UV^\top\|_\Omega^2$, where
$T\in\setR^{m\times n}$ is some target matrix, $U\in\setR^{m\times k}$
and $V\in\setR^{n\times k}$ are the low-rank factor matrices, and
$\Omega\in\{0, 1\}^{m\times n}$ is an indicator matrix that defines
which elements of $T$ are known. Matrix factorization is mainly used
in the context of recommender
systems~\cite{Koren09} or natural language
processing~\cite{Blei03,Hofmann99}. For the experiments, we set $k=5$ and compute
the gradient and Hessian with respect to $U$. Note that the Hessian is
a fourth order tensor. 

\paragraph{Neural Net.} We have created a small neural net with ten
fully connected layers, ReLU activation functions, and a softmax
cross-entropy output layer. The weight matrices for the different
layers all had the same size, $n\times n$. Here, we report running
times for computing the Hessian of the first layer. Note, that the ReLU activation function, i.e., $f(x) = \max\{0, x\}$ is non-differential at $x=0$. All automatic differentiation packages return a subgradient in this case. However, it has been shown recently~\cite{Lee2020} that this convention still guarantees correctness in a generalized sense.

\paragraph{Evaluation.} In the case of scalar-valued
functions all frameworks basically work in the same way. Thus,
it is not surprising that their running times for computing function
values and gradients are almost the same, see Figure~\ref{fig:ex1}.

The situation is different for Hessians. First, it can be seen that
the reverse mode in our approach, whose results agree with the results
in~\cite{LaueMG2018}, is a few orders of magnitude faster than current
state-of-the-art frameworks like TensorFlow, PyTorch, autograd, and
JAX.  This holds true for all experiments on the CPU and on the GPU,
see Figure~\ref{fig:ex2}.  For the logistic regression problem our
cross-country mode is about $30\%$ faster on the CPU, while its effect
on the GPU is negligible because of the GPU overhead. The performance
gain of our compression scheme can be seen on the matrix factorization
problem and on the neural net (middle and right column of
Figure~\ref{fig:ex2}). Computing Hessians for small neural nets has
now become feasible.

In our experiments, the effect of recent efforts in speeding up deep
learning frameworks turned out to be rather small. Actually, enabling
XLA for TensorFlow on the CPU slowed the computation down by a factor
of two. Hence, we omit these running times in
Figure~\ref{fig:ex2}. Enabling XLA on the GPU provided only marginal
improvements. JAX which relies on XLA did not finish computations but
raised memory errors indicating that it went out of main memory.
These errors seem to be caused by the automatic batching
function~\texttt{vmap} that is used by JAX for auto-vectorization when
computing Hessians. This was surprising to us since JAX is meant to be
more memory efficient than TensorFlow. There is one exception, on the
GPU, JAX finished the computation for the logistic regression
problem. However, as can be seen in Figure~\ref{fig:ex2}, even in this
case it is not significantly more efficient than the other deep
learning frameworks. This was to be expected since JAX relies on XLA
and the authors of XLA report a speed-up of only $15\%$ in the GPU
setting~\cite{XLAspeed}.

\section{Conclusion}

We have developed a simple, efficient and provably correct framework
for computing derivatives of general tensor expressions that is much
simpler than previous approaches. Furthermore, it can be easily
integrated into state-of-the-art frameworks like TensorFlow and
PyTorch that use the same tensor representation, but are a few orders
of magnitude slower than our approach when computing higher order
derivatives. We have also demonstrated that reverse mode automatic
differentiation is not optimal for computing higher order
derivatives. Significant speed ups can be achieved by a special
instantiation of the cross-country mode and by compressing higher
order derivatives. The algorithms presented here form the basis for the online tool \url{www.MatrixCalculus.org}~\cite{LaueMG19,Laue2019} for computing derivatives of matrix and tensor expressions. It is also used within the GENO framework~\cite{genoNIPS} for automatically generating optimization solvers from a given mathematical formulation. It can be accessed at \url{www.geno-project.org}~\cite{LaueMG20}.


\section*{Acknowledgments}
S\"oren Laue has been funded by Deutsche Forschungsgemeinschaft (DFG) under grant LA~2971/1-1.


\newpage
\appendix
\section*{Appendix}
For illustration purposes, Figure~\ref{fig:hessian} shows the computational graph of the Hessian
of a net with three fully connected layers each with a ReLU activation
function and a final cross-entropy layer that has been computed using
reverse mode. Nodes that represent fourth order tensors are marked in
red. Figure~\ref{fig:hessian} shows that is impossible to eliminate these nodes easily when using reverse mode. 
Most of the computation time is actually spend in evaluating
these fourth order tensors. However, using our instantiation of the
cross-country mode it becomes possible to compute the Hessian without
ever creating a fourth order tensor.  See Figure~\ref{fig:hessian2}
for the resulting computational graph. The only node that represents a
fourth order tensor can now be safely removed as described in
Section~3.3 of the main paper.

\vspace{2\baselineskip}

[Figures~\ref{fig:hessian} and~\ref{fig:hessian2} follow on the next two pages]

\begin{figure*}[h!]
\begin{center}
  \includegraphics[width=0.74\textwidth]{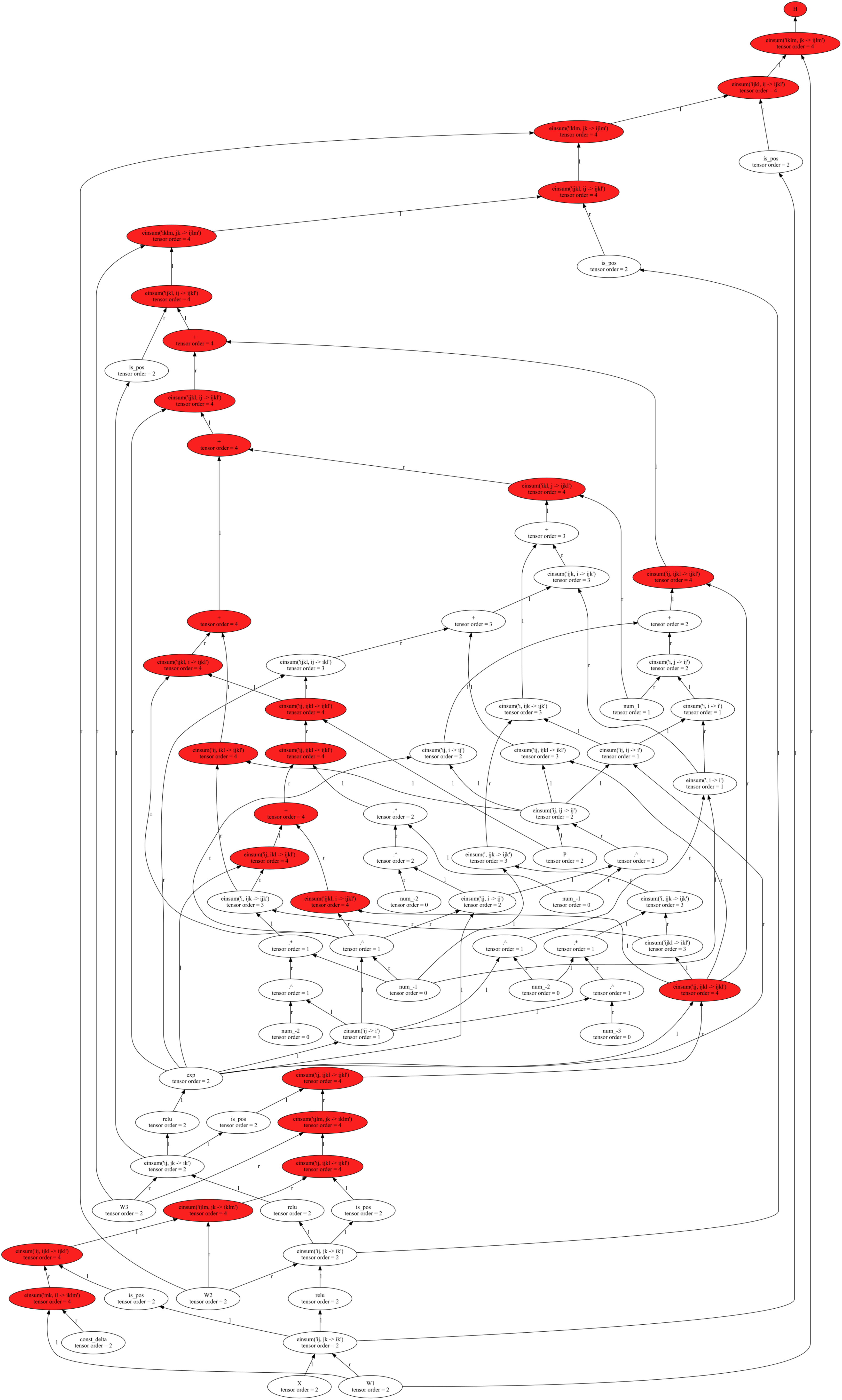}
  \caption{Hessian of a net with three fully connected layers, ReLU
    activation functions, and a cross-entropy layer when computed in
    reverse mode. Fourth order tensors are marked in red. Please note, the purpose of this figure is to illustrate that it is a non-trivial task to remove the fourth order tensor nodes when using reverse mode.}
  \label{fig:hessian}
\end{center}
\end{figure*}
\begin{figure*}[h!]
\begin{center}
  \includegraphics[width=0.68\textwidth]{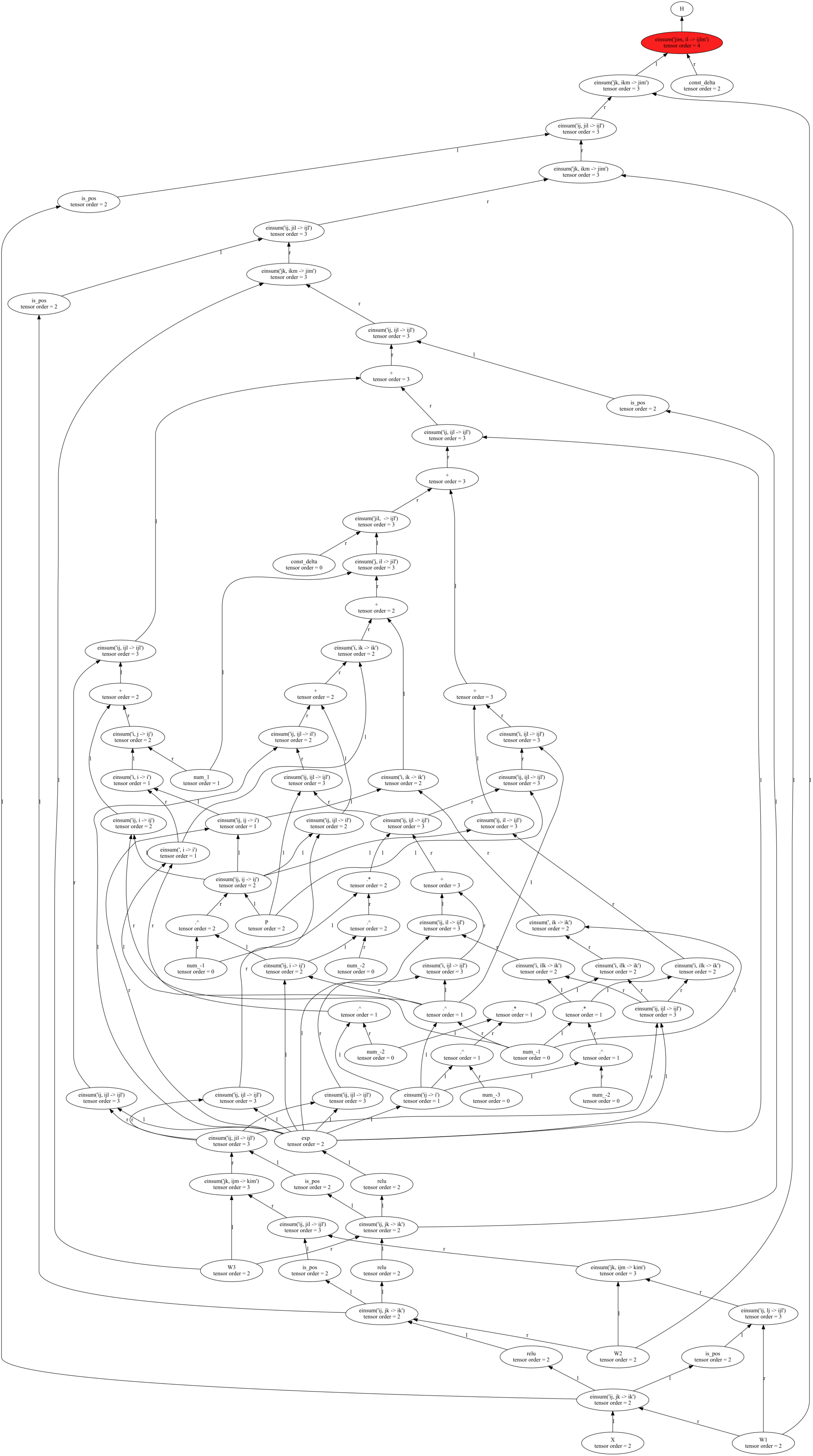}
  \caption{Hessian of a net with three fully connected layers, ReLU
    activation functions, and a cross-entropy layer when computed in
    our cross-country mode. The only node that represents a fourth
    order tensor, shown in red, can be easily removed due to using the cross-country mode.}
  \label{fig:hessian2}
\end{center}
\end{figure*}


\end{document}